\documentclass[nohyperref]{article}

\usepackage{microtype}
\usepackage{graphicx}

\usepackage{booktabs} 

\usepackage{hyperref}

 \usepackage[accepted]{icml2022}

\usepackage{amsmath}
\usepackage{amssymb}
\usepackage{mathtools}
\usepackage{amsthm}

\usepackage[capitalize,noabbrev]{cleveref}

\theoremstyle{plain}
\newtheorem{theorem}{Theorem}[section]

\theoremstyle{definition}
\newtheorem{definition}[theorem]{Definition}

\theoremstyle{remark}

\usepackage[textsize=tiny]{todonotes}

\usepackage[utf8]{inputenc} 
\usepackage[T1]{fontenc}    
\usepackage{hyperref}       
\usepackage{url}            
\usepackage{booktabs}       
\usepackage{amsfonts}       
\usepackage{nicefrac}       
\usepackage{microtype}      
\usepackage{xcolor}         

\usepackage{ amssymb }
\usepackage{ amsmath }
\usepackage{ amsfonts}
\usepackage{ dsfont }
\usepackage{ amsthm}
\usepackage{graphicx}
\usepackage{bbm}
\usepackage{ enumerate}
\usepackage{bm}
\usepackage{tikz}
\usetikzlibrary{backgrounds,automata, shapes}

\newcommand{\R}{\mathbb{R}}

\newcommand{\bmu}{{\boldsymbol \mu}}

\newcommand{\N}{\mathcal{N}}
\newcommand{\GP}{\mathcal{GP}}

\newcommand{\eqdef}{\overset{\text{\tiny{def}}}{=}}

\newcommand{\dist}{\text{proto}}

\usepackage{graphicx}
\usepackage{caption}
\usepackage{subcaption}

\usepackage{pifont}
\newcommand{\xmark}{\ding{55}}%

\begin{document}

\twocolumn[

\icmltitlerunning{Scalable First-Order Bayesian Optimization
via Structured Automatic Differentiation}

\icmltitle{Scalable First-Order Bayesian Optimization \\ 
via Structured Automatic Differentiation}



\icmlsetsymbol{equal}{*}

\begin{icmlauthorlist}
\icmlauthor{Sebastian Ament}{cornell} 
\qquad
\qquad
\icmlauthor{Carla Gomes}{cornell}
\end{icmlauthorlist}

\icmlaffiliation{cornell}{Department of Computer Science, Cornell University, Ithaca, NY, 14850, USA}
\icmlcorrespondingauthor{
	\href{https://sebastianament.github.io}{Sebastian Ament}
	}{\texttt{ament@cs.cornell.edu}}

\icmlkeywords{Machine Learning, ICML}

\vskip 0.3in
]

\printAffiliationsAndNotice{}  


\begin{abstract}
Bayesian Optimization (BO) has shown great promise for the global optimization of functions that are expensive to evaluate, but
despite many successes,
standard approaches 
can struggle in high dimensions.
To improve the performance of BO, 
prior work suggested incorporating gradient information into a Gaussian process surrogate of the objective, giving rise to kernel matrices of size $nd \times nd$ for $n$ observations in $d$ dimensions.
Na\"ively multiplying with (resp. inverting) these matrices requires $\mathcal{O}(n^2d^2)$ (resp. $\mathcal{O}(n^3d^3$)) operations, which becomes infeasible for moderate dimensions and sample sizes.
Here, we observe that a wide range of kernels gives rise to structured matrices,
enabling an {\it exact} $\mathcal{O}(n^2d)$ matrix-vector multiply for gradient observations
and $\mathcal{O}(n^2d^2)$ for Hessian observations.
Beyond canonical kernel classes, 
we derive a programmatic approach to leveraging this type of structure for transformations and combinations of the discussed kernel classes, which constitutes a structure-aware automatic differentiation algorithm. 
Our methods apply to 
virtually all canonical kernels and automatically extend to
complex kernels, like the neural network, radial basis function network, and spectral mixture kernels without any additional derivations,
enabling flexible, problem-dependent modeling
while scaling first-order BO to high $d$.
\end{abstract}

\section{Introduction}

Bayesian Optimization (BO) has demonstrated tremendous promise for the global optimization of functions, in particular those that are expensive to evaluate
\citep{shahriari2016bo, frazier2018tutorial}.
Instantiations of BO can be found in Active Learning (AL) \citep{settles2009active, tuia2011survey, fu2013survey}, the optimal design of experiments \citep{chaloner1995bayesian, foster2019design, zheng2020sequential}, and Optimal Learning \citep{powell2012optimal}.
Its applications range widely from the optimization of hyper-parameters of complex machine learning models \citep{snoek2012practical} 
to the sciences and engineering as \citet{attia2020closed}, who optimized charging protocols to maximize battery life.
\citet{li2018hyperband} reported that random search with only twice as many samples can outperform standard BO methods on a certain hyper-parameter optimization task.
This lead \citet{ahmed2016we} to advocate for first-order BO (FOBO) as a critical improvement, a call that recently received theoretical heft due to \citet{shekhar2021significance}, 
who
proved that
FOBO achieves an {\it exponential} improvement on
 the expected regret of standard BO for multi-armed bandit problems
 as a function of the number of observations $n$ and dimensionality $d$ of the input.

At the same time, differentiable programming and automatic differentiation (AD), which enable the calculation of gradients through complex numerical programs,
have become an integral part of machine learning research \citep{innes2017pl, wang2018backpropagation, baydin2018ad}
and practice, perhaps best illustrated by PyTorch \citep{paszke2019pytorch} and Tensorflow \citep{tensorflow2015whitepaper}, both of which include AD engines.
Certainly, AD has powered an increasing pace of model development by automating the error-prone writing of derivative code and is thus a natural complement to FOBO,
if only to compute the gradients of the objective.

On a high level, most BO approaches build a surrogate model of an objective with a few potentially noisy observations and make informed choices about further queries based on predictive values and uncertainties of the surrogate.
In principle, any functional form could be employed as a surrogate, and indeed
\citet{wang2014new}, \citet{snoek2015scalable}, and \citet{gal2017deep} use deep neural networks for AL and BO.
However, Gaussian Processes (GP) are currently the most commonly used models for research and applications of BO because they work well with little data and permit closed-form posterior inference.
Fortunately,
GPs are closed under differentiation 
with benign assumptions, see Section~\ref{sec:gp_differentiability},
and maintain their analytical properties when conditioned on gradient information \citep{solak2003derivative}.

Nonetheless, na\"ively incorporating gradients leads to kernel matrices of size $nd \times nd$, for $n$ observations in $d$ dimensions,
which restricts the possible problem sizes and dimensions, a problem that needs to be overcome to make FOBO applicable to a wide array of problems.
Further, as the performance of GPs chiefly depends on their covariance kernel,
it is important to give researchers and practitioners flexibility in this choice.
Herein, it is our primary goal to enable {\it scalabe inference} for GPs 
in the context of FOBO,
while maintaining {\it modeling flexibility} via matrix-structure-aware AD.

\paragraph{Contributions}

We 1) derive analytical block-data-sparse structures for a large class of gradient kernel matrices, allowing for an exact $\mathcal{O}(n^2d)$ multiply in Section~\ref{sec:methods},
2) propose an AD framework that programmatically computes the data-sparse block structures for transformations, and algebraic combinations of kernels
and make \href{https://github.com/SebastianAment/CovarianceFunctions.jl}{our implementation} publicly available\footnote{\href{https://github.com/SebastianAment/CovarianceFunctions.jl}{\scriptsize \bf \texttt{github.com/SebastianAment/CovarianceFunctions.jl}}}.
In Section~\ref{sec:hessstructure}, 
we further 3) derive analogous structures for kernel matrices that arise from conditioning on Hessian information, 
reducing the complexity from $\mathcal{O}(n^2d^4)$ for the na\"ive approach to 
$\mathcal{O}(n^2d^2)$,
4) provide numerical experiments that demonstrate the improved scaling and
delineate the problem sizes for which the proposed methods are applicable in Section~\ref{sec:scaling},
5) compare against existing techniques in Section~\ref{sec:comparison_existing} and 
6) use the proposed methods for Bayesian Optimization in Section~\ref{sec:test}.

\section{Related Work}
\paragraph{Gaussian Processes}
Inference for GPs has traditionally been based on matrix factorizations, but recently,
methods based on iterative solvers have been developed, which can scale up to a million data points without approximations \citep{wang2019exact}
by leveraging the parallelism of modern hardware \citep{dong2017scalable, gardner2018gpytorch}.
Extending the approximate matrix-vector multiplication algorithms of \citet{wilson2015kernel}
and \citet{gardner2018product}, 
\citet{eriksson2018scaling} proposed an approximate method for GPs with derivative information which scales quasi-linearly in $n$ for separable product kernels whose constituents are stationary. 
\citet{de2021high} proposed an elegant direct method for GPs with derivatives that scales linearly in the dimensionality but sextically -- $\mathcal{O}(n^6 + n^2d)$ -- with the number of data points and also derive an efficient multiply for dot-product and isotropic kernels whose inputs can be scaled by a diagonal matrix.
\citet{frazier2017gradients} used GPs with gradients for BO 
and proposed keeping only a single directional derivative to reduce the computational cost.
\citet{padidar2021scaling} proposed a similar strategy, retaining only relevant directional derivatives, to scale a variational inference scheme for GPs with derivatives.
Notably, incorporating gradient information into GPs is not only useful for BO:
\citet{solak2003derivative} put forward the integration of gradient information for GP models of dynamical systems,
\citet{riihimaki2010gaussian} used ``virtual'' derivative observations to include monotonicity constraints into GPs, and
\citet{solin2018magnetic} employed the derivatives of a GP to model curl-free magnetic fields and their physical constraints.

\paragraph{Automatic Differentiation}
To disambiguate several sometimes conflated terms,
we quote \citet{baydin2018ad}, who defined AD as ``a specific family of techniques that computes derivatives through accumulation of values during code execution to generate numerical derivative evaluations rather than derivative expressions''.
It enables the computation of derivatives up to machine precision 
while maintaining the speed of numerical operations.
Practical implementations of AD include forward-mode differentiation techniques based on operator overloading \citep{revels2016fd},  the $\partial P$ system of \citet{innes2018zygote},
which is able to generate compiled derivative code of differentiable components of the Julia language,
as well as the reverse-mode differentiation technologies of 
PyTorch \citep{paszke2019pytorch} and Tensorflow \citep{tensorflow2015whitepaper}.
\citet{maclaurin2015gradient} put forward an algorithm for computing gradients of models w.r.t. their {\it hyper-parameters} using reverse-mode auto-differentiation,
enabling the use of FOBO to optimize a model's generalization performance.
Among others, \citet{verma1998structured} explored the exploitation of structure, primarily sparsity, in the automatic computation of Jacobian and Hessian matrices. However, the existing work is not directly applicable here, since it does not treat the more general {\it data-sparse} structures of Section~\ref{sec:methods}.
For a review of automatic differentiation (AD) techniques, see \cite{griewank2008evaluating}.

\paragraph{Bayesian Optimization}
Bayesian Optimization (BO) has been applied to a diverse set of problems,
and of particular interest to the machine learning community is the optimization of hyper-parameters of complex models \citep{klein2017fast}.
Spurring much interest in BO, \citet{snoek2012practical} demonstrated that BO is an effective tool for the optimization of hyper-parameters of deep neural networks.
\citet{hennig2012entropy} proposed entropy search for global optimization,
a technique that employs GPs to compute a distribution over the potential optimum of a function.
\citet{wang2013bayesian} proposed efficient BO with random embeddings which scales to very high-dimensional problems by exploiting lower-dimensional structures.
\citet{krause2018boadditive} assumed an additive structure to scale BO to high dimensions.
\citet{eriksson2018scaling} used their fast approximate inference technique
for FOBO in combination with an active subspaces method \citep{constantine2014active} in order to reduce the dimensionality of the optimization problem and to speed up convergence.
\citet{martinez2018practical} enabled BO in the presence of outliers by employing a heavy-tailed likelihood distribution.
\citet{malkomes2018automating} used BO in model space to choose surrogate models for use in a primary BO loop.
\citet{frazier2019twostep} presented a 
two-step
 lookahead method for BO.
 \citet{eriksson2019turbo} put forth \texttt{TuRBO}, leveraging a set of local models for the global optimization of high-dimensional functions.
BO is also applied to hierarchical reinforcement learning \citep{brochu2010tutorial, prabuchandran2021boandrl}.
Existing BO libraries include 
\href{https://github.com/dragonfly/dragonfly}{Dragonfly} \citep{kandasamy2020dragonfly},
\href{https://github.com/rmcantin/bayesopt/}{BayesOpt} \citep{martinez2014bayesopt},
and \href{https://botorch.org}{BoTorch} \citep{balandat2020botorch}.
For a review of BO, see \citep{frazier2018tutorial}.

\section{Methods}
\label{sec:methods}

\subsection{Preliminaries}

We first provide definitions and set up notation 
and central quantities for the rest of the paper.
\begin{definition}
A random function $f$ is a Gaussian process with a mean $\mu$ and covariance function $k$ if and only if all of its finite-dimensional marginal distributions are multivariate Gaussian distributions.
In particular, $f$ is a Gaussian process if and only if
for any finite set of inputs $\{\bold x_i\}$,
\[
\bold f
\sim \N\left(
\bmu, 
\bold K
\right),
\]
where $f_i = f(\bold x_i)$, 
$\mu_{i} = \mu(\bold x_i)$ and $K_{ij} = k(\bold x_i, \bold x_j)$.
In this case, we write $f \sim \GP(\mu, k)$.
\end{definition}

When defining kernel functions, $\bold x$ and $\bold y$ will denote the first and second inputs,
$\bold r = \bold x - \bold y$ their difference,
$\bold I_d$ the $d$-dimensional identity matrix, 
and $\bold 1_d$ the all-ones vector of length $d$.
The gradient and Jacobian operators with respect to $\bold x$ will be denoted by
$\nabla_{\bold x}$ and $\bold J_{\bold x}$, respectively.

\paragraph{The $\bold G$ Operator}
 The focus of the present work is the matrix-valued operator $\bold G = \nabla_{\bold x} \nabla_{\bold y}^\top$ that acts on kernel functions $k(\bold x, \bold y)$ and whose entries are $G_{ij} = \partial_{x_i} \partial_{y_j}$.
We will show that $\bold G[k]$ is highly structured and data-sparse for a vast space of kernel functions
and present an automatic structure-aware algorithm for the computation of~$\bold G$.  
The kernel matrix $\bold K^{\nabla} = \bold G[k](\bold X)$ that arises from the evaluation of $\bold G[k]$ on the data $\bold X = [\bold x_1 \hdots \bold x_n]$
can be seen as a block matrix whose $(i, j)$th block is $\bold K^{\nabla}_{ij} = \bold G[k](\bold x_i, \bold x_j)$.
For isotropic and dot-product kernels, \citet{de2021high} discovered that $\bold K^\nabla$ has the structure
$\bold K^{\nabla} = k'(\bold X) \otimes \bold I_d + [\text{rank-$n^2$ matrix}]$,
which allows a linear-in-$d$ direct inversion, though the resulting $\mathcal{O}(n^6)$-scaling only applies to the low-data regime. 
Rather than deriving similar global structure, 
we focus on efficient structure for the blocks $\bold G[k](\bold x_i, \bold x_j)$, which is more readily amenable to a fully lazy implementation with $\mathcal{O}(1)$ memory complexity, and the synthesis of several derivative orders,
see Sec. \ref{sec:combiningorders} for details.
Last, we stress that our goal here is to focus on the {\it subset} of transformations that arise in most kernel functions, and {\it not} the derivation of a fully general structured AD engine for the computation of the $\bold G$ operator.

\subsection{Gradient Kernel Structure}
In this section, we derive novel structured representations of $\bold G[k]$ for a large class of kernels $k$.
The only similar previously known structures are for isotropic and dot-product kernels derived by \citet{de2021high}.

\label{sec:gradstructure}
\paragraph{Input Types}

The majority of canonical covariance kernels can be written as 
\[
\begin{aligned}
k(\bold x, \bold y) = f(\dist(\bold x, \bold y)),
\end{aligned}
\]
where
$ \dist(\bold x, \bold y) = (\bold r \cdot \bold r), \ (\bold c \cdot \bold r), \ \text{or} \ (\bold x \cdot \bold y)$,
$f$ is a scalar-valued function, and $\bold c \in \R^d$.
The first two types make up most of commonly used stationary covariance functions, 
while the last constitutes the basis of many popular non-stationary kernels.
We call the choice of \dist \ isotropic, stationary linear functional, and dot product, respectively.
First, we note that $\bold G[\dist]$ is simple for all three choices:
\[
\begin{aligned}
\bold G [\bold r \cdot \bold r] = -\bold I_d, \ \
\bold G [\bold c \cdot \bold r] = \bold 0_{d \times d}, \ \ \text{and} \ \
\bold G [\bold x \cdot \bold y] = \bold I_d.
\end{aligned}
\]
 
Kernels with the first and third input type are ubiquitous and include the 
exponentiated quadratic, rational quadratic, Mat\'ern, and polynomial kernels. 
An important example of the second type is the cosine kernel, 
which has been used to approximate stationary kernels \citep{rahimi2007random, lazaro2010sparse, gal2015improving}
and is also a part of the spectral mixture kernel \citep{wilson2013sm}.
In the following, we systematically treat most of the kernels and transformations 
in \cite{rasmussen2005gpml} to greatly expand the class of kernels for which structured representations are available.

\paragraph{A Chain Rule}
Many kernels can be expressed as $k = f \circ g$ where $g$ is scalar-valued.
For these types of kernels, we have
\[
\bold G[f \circ g] = (f' \circ g) \ \bold G[g] + (f'' \circ g) \ \nabla_{\bold x}[g] \nabla_{\bold y}[g]^\top.
\]
That is, $\bold G[f \circ g]$ is a rank-one correction to $\bold G[g]$.
If $\bold G[g]$ is structured with $\mathcal{O}(d)$ data, $\bold G[f \circ g]$ inherits this property.
As an immediate consequence,
$\bold G[k]$ permits a matrix-vector multiply in $\mathcal{O}(d)$ time for all isotropic, stationary, and dot-product kernels that fall under the categories outlined above.
However, there are combinations and transformations of these base kernels that give rise to
more complex kernels and enable more flexible, problem-dependent modeling. 

\paragraph{Sums and Products}
First, covariance kernels are closed under addition and multiplication.
If all summands or coefficients are of the the same input-type, the sum kernel has the same input type since $(f \circ \dist) + (g \circ \dist) = (f + g) \circ \dist$ and 
similarly for products,
so that no special treatment is necessary beside the chain rule above.
An interesting case occurs when we combine kernels of different input types or more complex composite kernels.
For $k= \sum_i^r k_i$, we trivially have 
$\bold G[k] = \sum_i^r  \bold G[k_i]$,
and so the complexity of multiplying with $\bold G[k]$ is $\mathcal{O}(dr)$.
For product kernels $
k(\bold x, \bold y) = g(\bold x, \bold y) h(\bold x, \bold y)$, we have
\[
\bold G[k] =
\bold G[g] h
+ g \bold G[h]
+ \nabla_{\bold x} [g] \ \nabla_{\bold y} [h]^\top
+ \nabla_{\bold x} [h] \ \nabla_{\bold y} [g]^\top,
\]
which is a rank-two correction to the sum of the scaled constituent gradient kernels elements -- $\bold G g$ and $ \bold G h$ -- and therefore only adds $\mathcal{O}(d)$ operations to the multiplication with the constituent elements.
In general, the application of $\bold G$ to a product of $r$ kernels $k = \prod_i^r k_i$ gives rise to a 
rank-$r$
correction to the sum of the constituent gradient kernels:
\begin{equation}
\label{eq:product}
\bold G[k] =
 \sum_{i=1}^r 
\bold G[k_i] p_i
+ \bold J_{\bold x} [\bold k]^\top \ \bold P \ \bold J_{\bold y}[\bold k],
\end{equation}
where $p_i = \prod_{j\neq i} k_i$ and $P_{ij} = \prod_{t \neq i, j} k_t$,
whose formation would generally be $\mathcal{O}(r^2)$.
However, if $k_i \neq 0$ for all~$i$, we have
$p_i = k / k_i$ and 
$\bold P = k \ \bold D_{\bold k}^{-1} \ (\bold 1_r \bold 1_r^\top - \bold I_r) \ \bold D_{\bold k}^{-1}$,
where $\bold k = [k_1, \hdots, k_r]$, and $\bold D_{\bold k}$ is the diagonal matrix with $\bold k$ on the diagonal.
A matrix-vector multiplication with \eqref{eq:product} can thus be computed in $\mathcal{O}(dr)$.
If $r \sim d$, the expression is generally not data-sparse unless the Jacobians are,
which is the case for the following special type of kernel product.

\paragraph{Direct Sums and Products}

Given a set of $d$ kernels $\{k_i\}$ each of which acts on a different input dimension, we can define their direct product (resp. sum) as
$k(\bold x, \bold y) = \prod_i k_i(x_i, y_i)$ (resp. $ \sum_i k_i(x_i, y_i)$), where $x_i$ corresponds to the dimension on which $k_i$ acts.
This separable structure gives rise to 
sparse differential operators $\bold G k$ and $\bold J_{\bold x} k$ that are zero except for
\[
[\bold G k_i]_{ii} = [\partial_{x_i}\partial_{y_i} k_i] \prod_{j\neq i} k_j,
\ \ \
\text{and}
\ \ \
[\bold J_{\bold x}\bold k]_{ii} = \partial_{x_i} k_i.
\]
For direct sums, $\bold G k$ is then simply diagonal:
$ \bold G_{ii} k = \partial_{x_i} \partial_{y_i} k_i$.
For direct products, substituting these sparse expressions into the general product rule \eqref{eq:product} above
yields a rank-one update to a diagonal matrix.
Therefore, the computational complexity of multiplying a vector with $\bold G [k](\bold x, \bold y)$ for separable kernels is $\mathcal{O}(d)$.
Notably, the above structure can be readily generalized for block-separable kernels, 
whose constituent kernels act on more than one dimension.
The $\mathcal{O}(d)$ complexity is also attained as long as every constituent kernel only applies to a constant number of dimensions as $d \to \infty$, or itself allows a multiply that is linear in the dimensionality of the space on which it acts.

\paragraph{Vertical Rescaling}
If $k(\bold x, \bold y) = f(\bold x) h(\bold x, \bold y) f(\bold y)$ for a scalar-valued $f$, then
\[
\begin{aligned}
\bold G [k](\bold x, \bold y) &=
f(\bold x) \bold G [h](\bold x, \bold y) f(\bold y) \ +
\\ 
\nabla_{\bold x} \begin{bmatrix}
 f(\bold x) &
k(\bold x, \bold y)
\end{bmatrix}
&
\begin{bmatrix}
h(\bold x, \bold y) & f(\bold y) \\
f(\bold x) & 0
\end{bmatrix}
\nabla_{\bold y} \begin{bmatrix}
f(\bold y) &
k(\bold x, \bold y)
\end{bmatrix}^\top
\end{aligned}
\]
Again, $\bold G[k]$ is a low-rank (rank two) correction to $\bold G[h]$.

\paragraph{Warping}
The so called ``warping'' of inputs to GPs is an important technique for the incorporation of non-trivial problem structure, especially of a non-stationary nature \citep{snelson2004warped, lazaro2012bayesian, marmin2018warped}.
In particular, given some potentially vector-valued
warping function $\bold u: \bold R^d \to \bold R^{r}$
a warped kernel can be written as $k(\bold x, \bold y) = h(\bold u(\bold x), \bold u(\bold y))$,  which leads to
\[
\bold G[k](\bold x, \bold y) = \bold J[ \bold u](\bold x)^\top \
\bold G[h](\bold u(\bold x), \bold u(\bold y)) 
\ \bold J[ \bold u](\bold y).
\]
We can factor out the Jacobian factors as block-diagonal matrices 
$\text{diag}(\bold J[\bold u](\bold X))_{ii} = \bold J[\bold u](\bold x_i)$ 
from the gradient kernel matrix $\bold K^{\nabla}$, leading to an efficient representation:
\[
\bold K^{\nabla} = \text{diag}(\bold J[\bold u](\bold X))^\top \ \bold H^{\nabla} \ \text{diag}(\bold J[\bold u](\bold X)).
\] 
Taking advantage of the above structure, the complexity of multiplication with the gradient kernel matrix can be reduced to $\mathcal{O}(n^2r + ndr)$, which is $\mathcal{O}(n^2d)$ for $n~>~d~\geq~r$.
Important examples of warping functions are energetic norms or inner products of the form
$\bold r^\top \bold E \bold r$ or $\bold x^\top \bold E \bold y$ for some positive semi-definite matrix $\bold E$.
In this case, we can factor $\bold E = \bold U^\top \bold U$ in a pre-computation that is independent of $n$ using a pivoted Cholesky decomposition using $\mathcal{O}(dr^2)$ operations for a rank $r$ matrix,
and let $\bold u(\bold x) = \bold U \bold x$,
so that $\bold J[\bold u] = \bold U$.
This gives rise to a Kronecker product structure in the Jacobian scaling matrix
$\text{diag}(\bold J[\bold u](\bold X)) = \bold I_n \otimes \bold U$,
and enables subspace search techniques for BO,
like the ones of \citet{wang2013bayesian},
\citet{eriksson2018scaling}, and \citet{kirschner2019adaptive},
to take advantage of the structures proposed here.
If $\bold E$ is diagonal as for automatic relevance determination (ARD), 
one can simply use $\bold U = \sqrt{\bold E}$, and the complexity of multiplying with $\bold K^\nabla$
is $\mathcal{O}(n^2d + nd)$.
Notably, the matrix structure and its scaling also extend to complex warping functions $\bold u$, like \citet{wilson2016deep}'s deep kernel learning model.

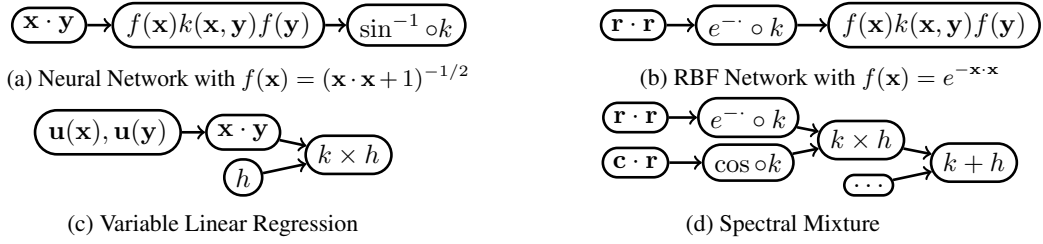
\begin{figure*}[t!]
\centering
\hspace*{\fill}
\subcaptionbox{Neural Network with $f(\bold x) = (\bold x \cdot \bold x + 1)^{-1/2}$}{%
\begin{tikzpicture}[scale=1]
\tikzstyle{every node}=[rounded rectangle, line width = 1., radius=0.5ex, draw] 
\node (a) at (0,1) {$\bold x \cdot \bold y$};
\node (b) at (2.25,1) {$f(\bold x) k(\bold x, \bold y) f(\bold y)$};
\node (c) at (4.75,1) {$\sin^{-1} \circ k$};
\draw [line width = 1.] (a) edge[->] (b)
(b) edge[->] (c);
\end{tikzpicture}
}
\hfill
\subcaptionbox{RBF Network with $f(\bold x) = e^{-\bold x \cdot \bold x}$}{%
\begin{tikzpicture}[scale=1.]
\tikzstyle{every node}=[rounded rectangle, line width = 1., radius=0.5ex, draw] 
\node (a) at (0,1) {$\bold r \cdot \bold r$};
\node (b) at (1.5,1) {$e^{-\cdot} \circ k$};
\node (c) at (4,1) {$f(\bold x) k(\bold x, \bold y) f(\bold y)$};
\draw [line width = 1.] (a) edge[->] (b)
(b) edge[->] (c);
\end{tikzpicture}
}
\hspace*{\fill}
\\
\hspace*{\fill}
\subcaptionbox{Variable Linear Regression}{%
\begin{tikzpicture}[scale=1.]
\tikzstyle{every node}=[rounded rectangle, line width = 1., radius=0.5ex, draw] 
\node (a) at (0,1) {$\bold u(\bold x), \bold u(\bold y)$};
\node (b) at (1.8,1) {$\bold x \cdot \bold y$};
\node (c) at (3.2,.7) {$k \times h$};
\node (d) at (1.8,.4) {$h$};
\draw [line width = 1.] (a) edge[->] (b)
(b) edge[->] (c) 
(d) edge[->] (c) ;
\end{tikzpicture}
}
\hfill
\subcaptionbox{Spectral Mixture}{%
\qquad
\begin{tikzpicture}[scale=1.]
\tikzstyle{every node}=[rounded rectangle, line width = 1., radius=0.5ex, draw] 
\node (a) at (0,1) {$\bold r \cdot \bold r$};
\node (b) at (1.5,1) {$e^{-\cdot} \circ k$};
\node (c) at (3,.7) {$k \times h$};
\node (d) at (0,.4) {$\bold c \cdot \bold r$};
\node (e) at (1.5,.4) {$\cos \circ k$};
\node (f) at (4.5,.4) {$k + h$};
\node (g) at (3.1,.1) {$\hdots$};
\draw [line width = 1.] (a) edge[->] (b)
(b) edge[->] (c) 
(d) edge[->] (e)
(e) edge[->] (c)
(c) edge[->] (f)
(g) edge[->] (f);
\end{tikzpicture}
}
\hspace*{\fill}
\caption{Computational graphs of composite kernels whose gradient kernel matrix can be expressed with the data-sparse structured expressions derived in Section \ref{sec:gradstructure}.
Inside a node, $k$ and $h$ refer to kernels computed by previous nodes.
}
\label{fig:compgraphs}
\end{figure*} 

\paragraph{Composite Kernels}

Systematic application of the rules and data-sparse representations of $\bold Gk$ for the transformations and compositions of kernels above gives rise to similar representations for many more complex kernels.
Examples include the neural network kernel 
$\arcsin(\tilde{\bold x} \cdot \tilde{\bold y})$, where $\tilde{\bold x} = \bold x / \sqrt{\|\bold x\|_2^2 + 1}$,
the RBF-network kernel 
$\exp(- \|\bold x\|^2 -\bold r \cdot \bold r/2 - \|\bold y\|^2)$,
the spectral mixture kernel of \citet{wilson2013sm}, and
the kernel $\boldsymbol \phi(\bold x)^\top \bold W \boldsymbol \phi(\bold y) h(\bold x, \bold y)$ corresponding to a linear regression with variable coefficients,
where $\boldsymbol \phi(\bold x)$ are the regression features, $\bold W$ is the prior covariance of the weights, and $h$ is a secondary kernel controlling the variability of the weights
\citep{rasmussen2005gpml}. 
See Figure~\ref{fig:compgraphs} for a depiction of these kernels' computational graphs,
where each node represents a computation that we treated in this section.
These examples highlight the generality of the proposed approach,
since it applies {\it without specializations} to these kernels,
and is simultaneously the first to enable a linear-in-$d$ multiply with their gradient kernel matrices $\bold K^\nabla$.

\subsection{Hessian Kernel Structure}
\label{sec:hessstructure}

Under appropriate differentiability assumptions (see Sec.~\ref{sec:gp_differentiability}), we can condition a GP on Hessian information.
However, incorporating second-order information into GPs has so far -- except for one and two-dimensional test problems by \citet{wu2017exploiting} -- not been explored.
 This is likely due to the prohibitive $\mathcal{O}(n^2 d^4)$ scaling for a matrix multiply with the associated covariance matrix and $\mathcal{O}(n^3d^6)$ scaling for direct matrix inversion. 
In addition to the special structure for the gradient-Hessian cross-covariance, already reported by \citet{de2021high},
we derive a structured representation of the Hessian-Hessian covariance for isotropic kernels, enabling efficient computations with second-order information.
In particular, letting $\bold h_{\bold x} = \text{vec}(\bold H_{\bold x})$ where $\bold H_{\bold x}$ is the Hessian w.r.t. $\bold x$ and $r = \bold r \cdot \bold r$:
\[
\begin{aligned}
\bold h_{\bold x} \nabla_{\bold y}^\top k (\bold x, \bold y) &= 
-f''(r) 
(\bold I_d \otimes \bold r + \bold r \otimes \bold I_d)  
\\
&- [f''(r) \text{vec}(\bold I_{d}) + f'''(r) \text{vec}(\bold r \bold r^\top)]
\bold r^\top, \ \text{and} \\
\bold h_{\bold x} \bold h_{\bold y}^\top k (\bold x, \bold y) &= 
(\bold I_{d^2} + \bold S_{dd}) [f''(r) \bold I_{d^2} \\
&+ 
f'''(r) (\bold r \bold r^\top \oplus \bold r \bold r^\top) ]  + \bold V \bold C \bold V^\top,
\end{aligned}
\]
where 
$\bold V = \begin{bmatrix}
\text{vec}(\bold I_{d}) & \text{vec}(\bold r \bold r^\top)
\end{bmatrix} \in \R^{d^2 \times 2}$,
$\bold C \in \R^{2 \times 2}$ with 
$\bold C_{ij} = \partial^{(i+j)} f(r)$,
$\bold S_{dd}$ is the ``shuffle'' matrix that satisfies
$\bold S_{dd} \text{vec}(\bold A) = \text{vec}(\bold A^\top)$,
and
$\bold A \oplus \bold B
=
\bold A \otimes \bold I + \bold I \otimes \bold B$ is the Kronecker sum.
Thus, it is possible to multiply with covariance matrices that arise from conditioning on second-order information
in $\mathcal{O}(n^2d^2)$, which is {\it linear} in the $\mathcal{O}(d^2)$ amount of information contained in the Hessian matrix and therefore optimal with respect to the dimensionality.
This is an attractive complexity in moderate dimensionality since Hessian observations are highly informative of a function's local behavior.
For derivations of the second-order covariances for more kernel types and transformations,
see Sec.~\ref{sec:hessian_appendix}.

\subsection{An Implementation: 
\href{https://github.com/SebastianAment/CovarianceFunctions.jl}{\texttt{CovarianceFunctions.jl}}}
\label{sec:implementation}

To take advantage of the analytical observations above in an {\it automatic} fashion, several technical challenges need to be overcome.
First, we need a representation of the computational graph of a kernel function that is built from the basic constituents and transformations that we outlined above, akin to Figure \ref{fig:compgraphs}.
Second, we need to build matrix-free representations of the gradient kernel matrices to maintain data-sparse structure.
Here, we briefly describe how we designed \texttt{CovarianceFunctions.jl},
an implementation of the structured AD technique that is enabled by the analytical derivations above, 
and supporting libraries, all written in Julia \citep{bezanson2017julia}.

\texttt{CovarianceFunctions.jl} represents kernels at the level of user-defined types.
It is in principle possible to hook into the abstract syntax tree (AST) to recognize these types of structures more generally \citep{innes2018zygote},
but this would undoubtedly come at the cost of increased complexity.
It is unclear if this generality would have applications outside of the scope of this work.
A user can readily extend the framework with a new kernel type 
if it can not already be expressed as a combinations and transformations of existing kernels.
All that is necessary is the definition of its evaluation and the following short function:
\texttt{input\_trait} returns the type of input the kernel depends on: isotropic, dot-product, or the stationary type $\bold c \cdot \bold r$ and automatically detects homogeneous products and sums of kernels with these input types.
As an example, 
for the rational quadratic kernel, we have
\[
\texttt{input\_trait(RQ$_\alpha$)} = \texttt{Isotropic()}
\]

\begin{figure*}[t!]
  \centering
\includegraphics[width=.9\textwidth]{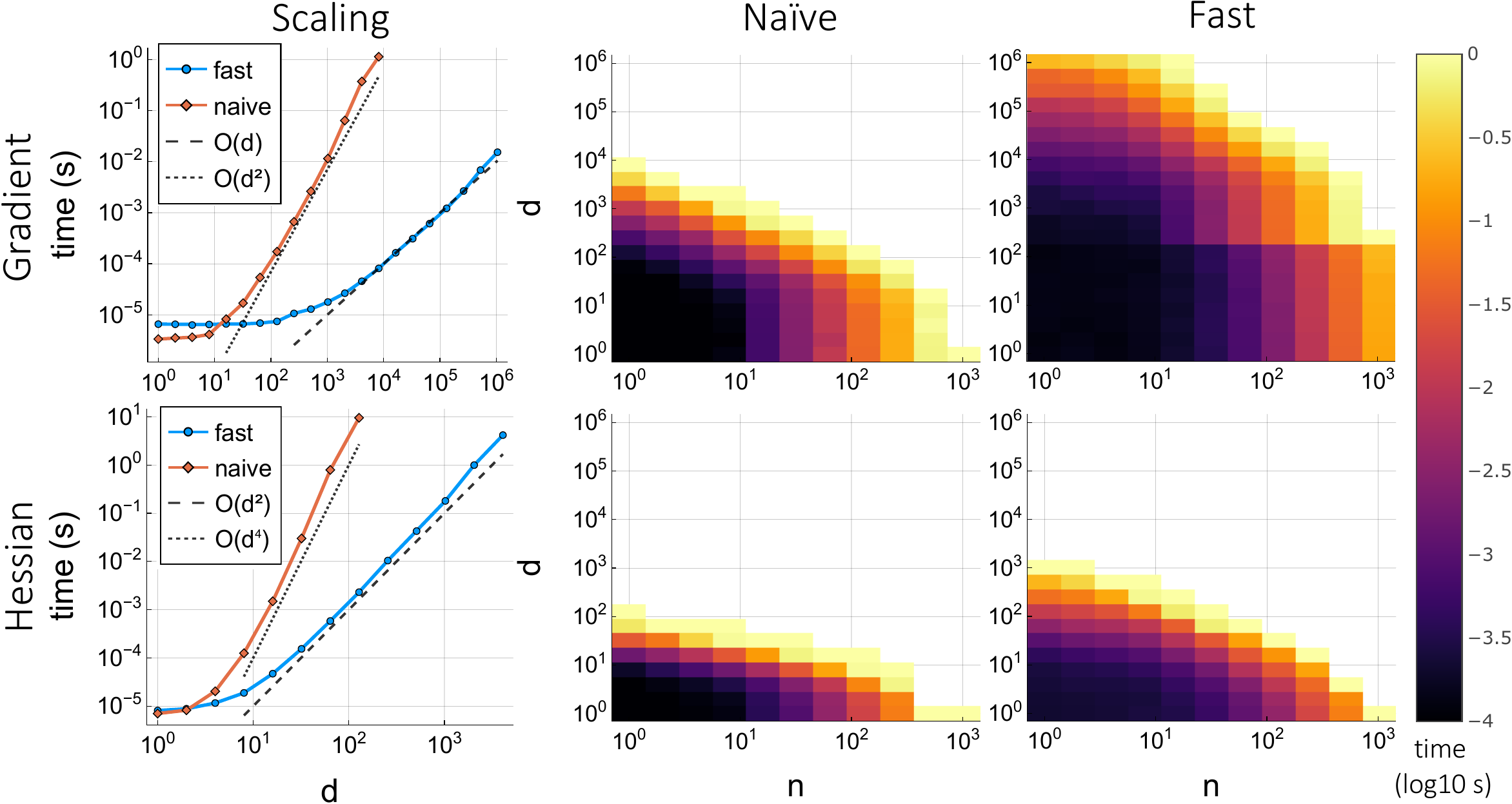}
  \caption{Benchmarks of matrix-vector-multiplications with the gradient (top) and Hessian kernel matrices (bottom) using a rational quadratic kernel. The scaling experiments (left) exhibit the predicted $\mathcal{O}(d)$ (resp. $\mathcal{O}(d^2)$) scaling of the fast algorithm for the gradient (resp. Hessian) kernel matrix with $n = 1$.
  The heat maps for the na\"ive (middle) and fast (right) algorithms show the execution time (color) as a function of $n$ and $d$.
  The fast methods exhibit a much larger region of sub-second run-times in the $n$-$d$ space, a proxy for the efficient applicability of the methods to problems of a particular size. 
  Note that both axes are exponential; even the visually modest improvements for the Hessian allow between one to two orders of magnitude higher dimensionality than the na\"ive approach given the same run-time.}
  \label{fig:scaling_rq}
\end{figure*}

Our implementation uses \texttt{ForwardDiff.jl} \cite{revels2016fd} to compute the regular derivatives and gradients that arise in the structured expressions to achieve a high level of generality
and for a robust fall-back implementation of all the relevant operators in case no structure can be inferred in the input kernel.
Even though the memory requirements of the $n^2$ data-sparse blocks are much reduced to the dense case, 
a machine can nevertheless run out of memory if the number of samples $n$ gets very large and all blocks are stored in memory.
To scale the method up to very large $n$, our implementation employs {\it lazy evaluation} of the gradient kernel matrix to achieve a constant, $\mathcal{O}(1)$, memory complexity for a matrix-vector multiply. %

The main benefit of this system is that researchers and practitioners of BO
{\it do not need to derive a special structured representation} for each kernel they want to use for an accurate modeling of their problem.
As an example, the structured AD rules of Section \ref{sec:gradstructure} obviate
the special derivation of the neural network kernel in Section \ref{sec:nn},
In our view, this has the potential to greatly increase the uptake of first-order BO techniques
outside of the immediate field of specialists.

\section{Experiments}
\label{sec:experiments}

\subsection{Scaling on Synthetic Data}
\label{sec:scaling}

First, we study the practical scaling of our implementation of the proposed methods with respect to both dimension and number of observations.
See Figure~\ref{fig:scaling_rq} for experimental results using the non-separable rational quadratic kernel.
Importantly, we observe virtually the same scaling behavior for more complex kernels like the neural network kernel.
Further, we stress that the scaling results are virtually indistinguishable for different kernels,
see also Figure~\ref{fig:scaling_nn} for similar experiments using the exponentiated dot-product and more complex neural network kernel.

 The exponentiated dot-product kernel $\exp(\bold x \cdot \bold y)$ was recently used by \citet{karvonen2021taylor} to derive a probabilistic Taylor-type expansion of multi-variate functions.
 The neural network kernel is derived as the limit of a neural network with one hidden layer, as the number of hidden units goes to infinity \citep{rasmussen2005gpml}.
The scaling plots on the left of Figure \ref{fig:scaling_rq} and \ref{fig:scaling_nn} were created with a single thread to minimize constants, while the heat-maps on the right were run with 24 threads on 12 cores in parallel to highlight the applicability of the methods on a modern parallel architecture.
 
 \begin{figure*}[t!]
  \centering
\includegraphics[width=.9\textwidth]{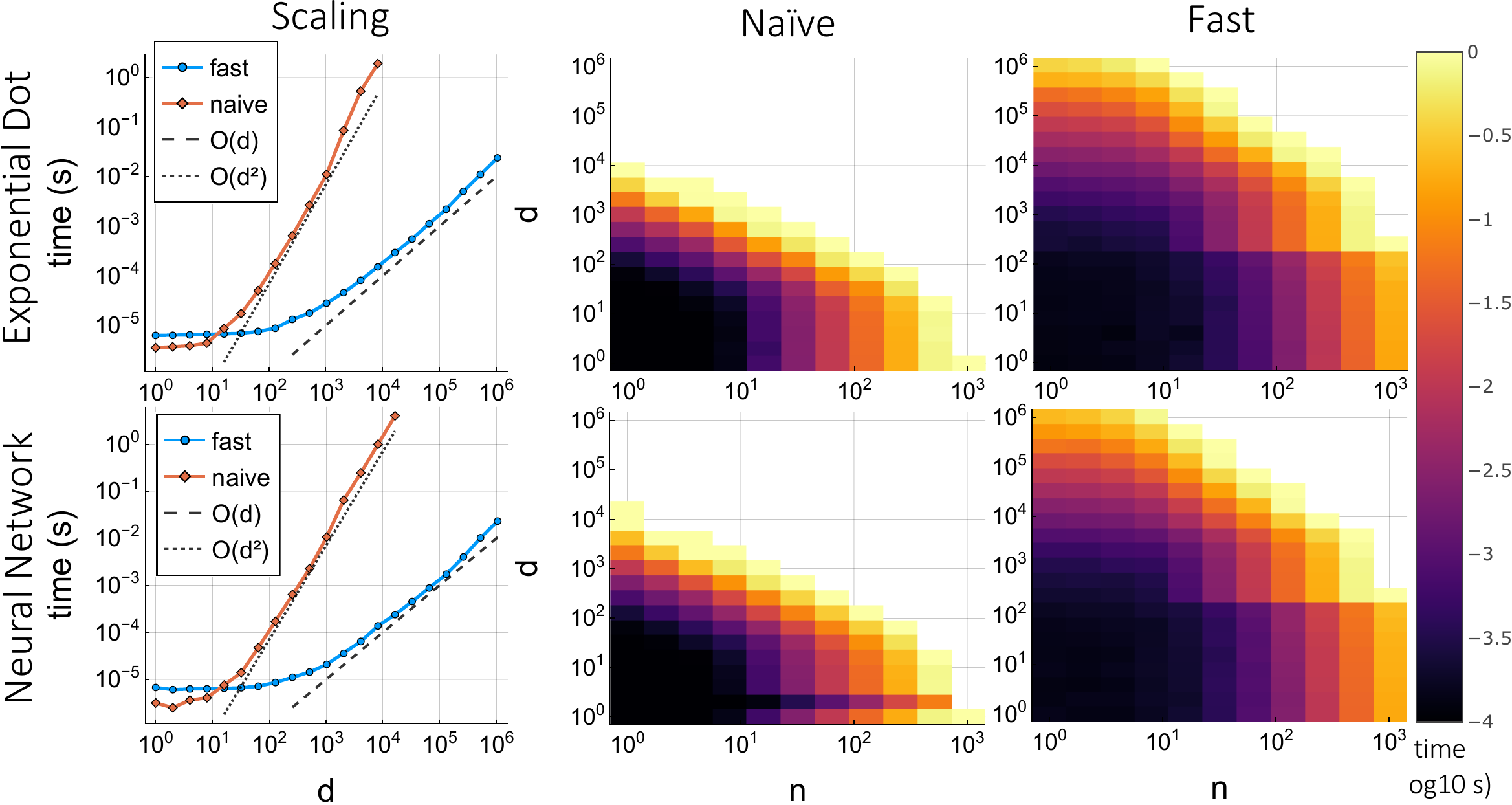}
  \caption{Benchmarks of matrix-vector multiplications with 
  the gradient kernel matrices using the exponentiated dot-product (top) and neural network kernels (bottom).
  Compared to the performance of the rational quadratic kernel in Figure~\ref{fig:scaling_rq}, the results for the composite neural network kernel are virtually indistinguishable and also exhibit the fast $\mathcal{O}(d)$ scaling. 
 }
  \label{fig:scaling_nn}
\end{figure*}

\begin{figure}[t!]
  \centering
\includegraphics[width=.48\textwidth]{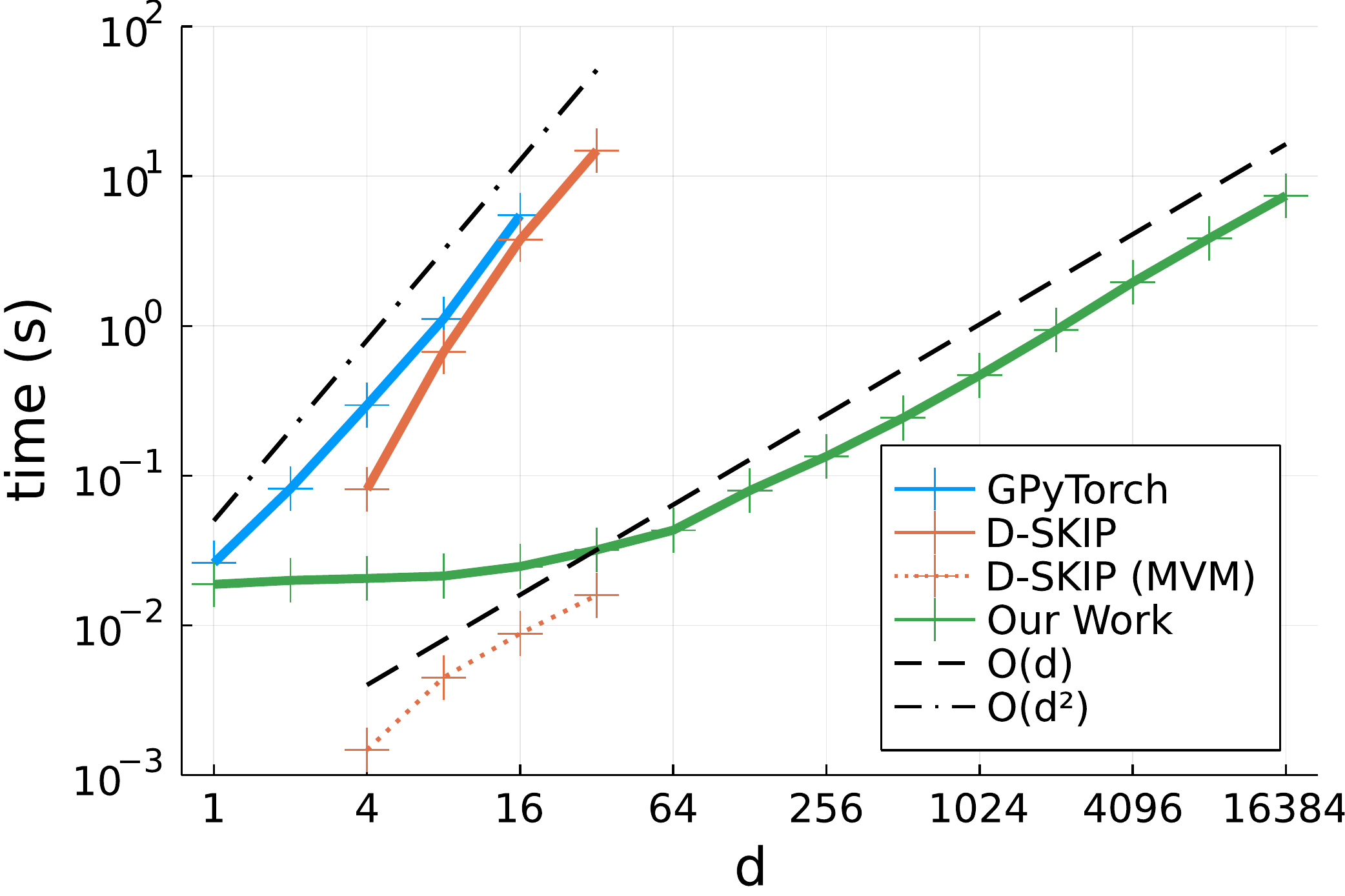}
  \caption{Time to first MVM of GPyTorch, D-SKIP, and our work for RBF gradient kernel matrices with $n = 1024$.}
  \label{fig:mvm_time_comparisons}
\end{figure}

\subsection{Comparison to Prior Work}
\label{sec:comparison_existing}
\paragraph{Existing Libraries}
While popular libraries like 
\href{https://gpytorch.ai}{GPyTorch}, 
\href{https://www.gpflow.org}{GPFlow},
and 
\href{https://scikit-learn.org/stable/}{Scikit-Learn}
have efficient implementations for the generic GP inference problem,
they do {\it not} offer efficient inference {\it with gradient observations},
see Table~\ref{table:gradient_kernel_comparison}
\cite{gardner2018gpytorch, matthews2017gpflow, scikit-learn}.
Highlighting the novelty of our work, GPyTorch 
contains only two implementations for this case
-- \href{https://docs.gpytorch.ai/en/stable/kernels.html?highlight=RBFKernelGrad#rbfkernelgrad}{RBFKernelGrad} 
and 
\href{https://docs.gpytorch.ai/en/stable/kernels.html?highlight=PolynomialKernelGrad#polynomialkernelgrad}{PolynomialKernelGrad} --
both with the na\"ive $\mathcal{O}(n^2d^2)$ matrix-vector multiplication complexity,
 hand-written work that is both obviated and outperformed by our structure-aware AD engine,
 see Figure~\ref{fig:mvm_time_comparisons}.
 Thus, \href{https://botorch.org}{BoTorch}, which depends on GPyTorch, does not yet support efficient FOBO.
Neither GPFlow nor 
SKLearn contain any implementations of gradient kernels.
\href{https://github.com/dragonfly/dragonfly}{Dragonfly} 
and
\href{https://github.com/rmcantin/bayesopt/}{BayesOpt}
do not support gradient observations.

\begin{table}[h!]
\centering
  \caption{\small MVM complexity with select gradient kernel matrices.
  SM = spectral mixture kernel, NN = neural network kernel. \\
  $^*$See the discussion on the right about D-SKIP's complexity.}
\label{table:gradient_kernel_comparison}
  \begin{tabular}{c c c c c}
    \toprule
    	 & RBF & SM & NN \\
	\midrule
	GPFlow / SKLearn & \xmark & \xmark & \xmark \\
	GPyTorch & $\mathcal{O}(n^2d^2)$ & \xmark & \xmark\\ 
	\citep{eriksson2018scaling} & $\mathcal{O}(nd^2)^*$ & \xmark & \xmark \\  
	\citep{de2021high} & $\mathcal{O}(n^2d)$ & \xmark & \xmark \\
	\href{https://github.com/SebastianAment/CovarianceFunctions.jl#gradient-kernels}{Our work}
	 & $\mathcal{O}(n^2d)$ & $\mathcal{O}(n^2d)$ & $\mathcal{O}(n^2d)$ \\
    \midrule
    \bottomrule
  \end{tabular}
\end{table}

\paragraph{\citet{eriksson2018scaling}'s D-SKIP} 
D-SKIP is an {\it approximate} method
and requires
that the kernel can be expressed as a separable product
and further, that the resulting constituent kernel matrices have a low rank.
In contrast our method is {\it mathematically exact}
and applies to a large class of kernels without restriction.
D-SKIP needs an upfront cost of $\mathcal{O}(d^2(n+m \log m+r^3n \log d))$,
followed by a matrix-vector multiplication (MVM) cost of $\mathcal{O}(dr^2n)$ for constituent kernel matrices of rank $r$
and $m$ inducing points per dimension.
For a constant rank $r$, D-SKIP's MVM scales both linearly in $n$ {\it and} $d$, while the method proposed herein 
scales quadratically in $n$.
See Figure~\ref{fig:mvm_time_comparisons} for a comparison of D-SKIP's real-world performance,
where D-SKIP's MVM scales linearly in $d$, but the required pre-processing scales quadratically in $d$ and dominates the total runtime.
Note that 
\href{https://github.com/ericlee0803/GP_Derivatives/blob/44934029a5bd8679daa2b2beb4018a30a108582a/code/kernels/se_kernel_grad_skip.m#L28}{D-SKIP's implementation}
is restricted to $d > 4$,
since D-SKI is faster in this regime.
For $d \leq 32$, D-SKIP's pure MVM times are faster than our method,
whose runtime grows sublinearly until $d = 64$ because it takes advantage of vector registers and SIMD instructions.
Notably, the linear extrapolation of D-SKIP's pure MVM times without pre-processing
is within a small factor ($<2$) of the timings of our work for $d \geq 64$,
implying that if D-SKIP were applied to higher dimensions,
the pure MVM times of both methods would be comparable for a moderately large number of observations ($n = 1024$).
Figure~\ref{fig:mvm_accuracy_comparisons} in Section~\ref{sec:accuracy_dskip} shows that D-SKIP is approximate and looses accuracy as $d$ increases, while our method is accurate to machine precision.

 \begin{figure*}[t!]
  \centering
\includegraphics[width=.98\textwidth]{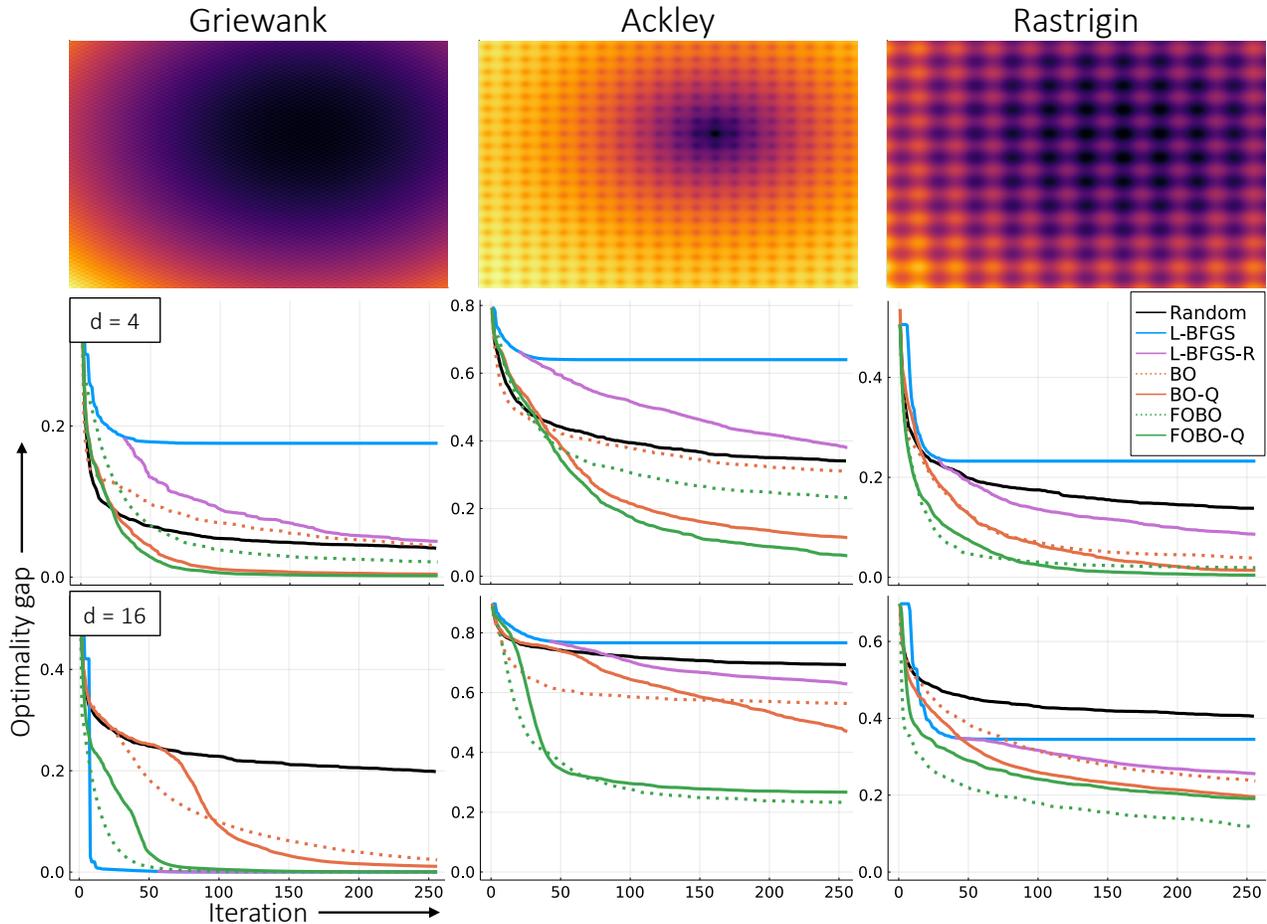}
  \caption{Average optimality gap of optimization algorithms on three non-convex test functions:
  Griewank, Ackley, and Rastrigin (plotted in 2D in top row).
  We compare random sampling (black), L-BFGS (blue), L-BFGS with random restarts after local convergence is detected (L-BFGS-R in purple), BO (dotted orange), BO with the quadratic mixture kernel of Section~\ref{sec:test} (BO-Q in solid orange), 
  FOBO (dotted green), and FOBO with the quadratic mixture kernel (FOBO-Q in solid green).
  All the BO variants use the expected improvement acquisition function
  and each line is the average optimality gap over 128 independent experiments.
  Notably, FOBO-Q outperforms on all 4D problems, 
  L-BFGS converges most rapidly on the 16D Griewank function because its many local minima vanish as $d$ increases so that the purely local search results in the fastest global convergence,
  and FOBO achieves the best optimality gap on the 16D Ackley and Rastrigin functions.
 }
  \label{fig:bo_benchmark}
\end{figure*}

\subsection{Bayesian Optimization}
\label{sec:test}

\citet{shekhar2021significance} proved that gradient information can lead to an exponential improvement in regret for multi-armed bandit problems, compared to zeroth-order BO,
by employing a two-stage procedure, the first of which hones in to a locally quadratic optimum of the objective.
Inspired by this result and studying the qualitative appearance of many test functions of \citet{bingham2013test}, a promising model for these objectives $f$ is a sum of functions $f = g + h$, where $g$ is a quadratic function and $h$ is a potentially quickly varying, non-convex function.
Since $h$ is arbitrary, the model does not restrict the space of functions,
but offers a useful inductive bias for problems with a globally quadratic structure ``perturbed'' by a non-convex function.
Assuming the minimum of the quadratic function coincides or is close to a global minimum to the objective, this structure can be exploited to accelerate convergence.

Herein, we model $g$ with a GP with the kernel $(\bold x \cdot \bold y + c)^2$, a distribution over quadratic functions whose stationary points are regularized by $c$,
while $h$  is assumed to be drawn from a GP with a Mat\'ern-5/2 kernel $k$,
to model quickly varying deviations from $g$.
Then $f = g + h \sim \GP(0, k(\bold x, \bold y) + (\bold x \cdot \bold y + c)^2)$. 
Notably, the resulting kernel is a composite kernel with isotropic and dot-product constituents, and a quadratic transformation.
Without exploiting this structure automatically as proposed above, one would have to derive a new fast multiply by hand, slowing down the application of this model to BO.

Notably, this model is similar to the one employed by \citet{gal2017deep},
who used a quadratic function as the prior mean, requiring a separate optimization or marginalization of the location and covariance of the mean function.
In contrast, we model the quadratic component with a specialized kernel, whose treatment only requires linear operations.

We benchmark both Bayesian and canonical optimization algorithms with and without gradient information on some of the test functions given by \citet{bingham2013test},
namely, the Griewank, Ackley, and Rastrigin functions.
See Section~\ref{sec:test_appendix} for the definitions of the test functions.
For all functions, 
we scaled the input domains to lie in $[-1, 1]^d$, scaled the output to lie in $[0, 1]$,
and shifted the global optimum of all functions to $\bold 1_d / 4$.
The top row of Figure~\ref{fig:bo_benchmark} shows plots of the non-convex functions in two dimensions.

Figure~\ref{fig:bo_benchmark} shows the average optimality gap over 128 independent experiments in four and sixteen dimensions for the following strategies:
random sampling (black), L-BFGS (blue), L-BFGS with random restarts after local convergence is detected (L-BFGS-R in purple), BO (dotted orange), BO with the quadratic mixture kernel of Section~\ref{sec:test} (BO-Q in solid orange), 
  FOBO (dotted green), and FOBO with the quadratic mixture kernel (FOBO-Q in solid green).
  The FOBO variants incorporate both value and gradient observations, see Section~\ref{sec:combiningorders}.
  All the BO variants use the expected improvement acquisition function
  which is numerically optimized w.r.t. the next observation point using L-BFGS.
  If the proposed next observation lies within $10^{-4}$ of any previously observed point,
  we choose a random point instead (see Algorithm~\ref{alg:bo}), similar to the L-BFGS-R strategy.
  This helps escape local minima in the acquisition function and improves the performance of all BO algorithms.

\begin{algorithm}[t!]
\caption{Bayesian Optimization with Restarts}
\label{alg:bo}
\begin{algorithmic}[1]
\STATE {\bfseries Input:} acquisition $a$, objective $f$, prior $p \sim \GP$
\STATE {\bfseries Ouput:} potential minimizer $\bold x^*$
\STATE initialize empty $\bold X$, $\bold y$ \\
\WHILE {budget not exhausted}
	\STATE $c \leftarrow p \ \big | \ \bold X, \bold y$  \hfill \COMMENT{compute conditional process} \\
	\STATE $\bold x_t \leftarrow \text{\it local} \arg \min_{\bold z} a(c, \bold z)$  \hfill \COMMENT{L-BFGS started at $\bold x_t$} \\
	\IF { $\min_{\bold z \in \bold X} \| \bold x_t - \bold z \| < \epsilon$} 
		\STATE $\bold x_t \leftarrow$ random point in domain of $f$
	\ENDIF
	\STATE append $\bold x$ to $\bold X$ and $f(\bold x_t)$ to $\bold y$
\ENDWHILE
\STATE $\bold x^* \leftarrow \arg \min_{\bold z \in \bold X} f(\bold z)$
\end{algorithmic}
\end{algorithm}

  FOBO-Q outperforms on all 4D problems, 
  L-BFGS converges most rapidly on the 16D Griewank function because its many local minima vanish as $d$ increases so that the purely local search results in the fastest convergence,
  and FOBO achieves the best optimality gap on the 16D Ackley and Rastrigin functions.
While the Rastrigin function contains a quadratic component analytically, its inference appears to become more difficult as the dimension increases, leading FOBO to outperform the Q variant.
But surprisingly, the Q variants outperform on the Ackley function for $d=4$,
even though it does {\it not} contain a quadratic component.

 \section{Conclusion}

\paragraph{Limitations} 
Incorporating gradient information improves the performance of BO,
 but the global optimization of non-convex functions remains {\it NP-hard} (see Section~\ref{sec:optimization_theory}) and can't be expected to be solved in general.
 For example, the optimum of the 16D Ackley function is elusive for all methods,
likely because its domain of attraction shrinks exponentially with $d$.
While we derived structured representations for kernel matrices arising from Hessian observations, we primarily focused on first-order information.
We demonstrated the improved computational scaling and feasibility of computing with Hessian observations in our experiments but did not use this for BO.
We leave a more comprehensive comparison of first and second-order BO to future work.
Our main goal here was to enable such investigations in the first place, by providing the required theoretical advances,
and practical infrastructure through \href{https://github.com/SebastianAment/CovarianceFunctions.jl}{\texttt{CovarianceFunctions.jl}}.

\paragraph{Future Work}

1) We are excited at the prospect of linking
\citet{maclaurin2015gradient}'s algorithm for the computation of hyper-parameter gradients
with the technology prosed here, {\it enabling efficient FOBO of hyper-parameters}.
2) While the methods proposed here are exact and enable a linear-in-$d$ MVM complexity, $\mathcal{O}(n^2)$ can still become expensive with a large number of observations $n$.
We believe that analysis-based fast algorithms like \citet{ryan2022fast}'s Fast Kernel Transform
could be derived for gradient kernels and hold promise in low dimensions.
Further, BO trajectories can yield redundant information particularly when honing in on a minimum,
which could be exploited using sparse linear solvers like the ones of \citet{ament2021sparse}.
3)~Hessian observations could be especially useful for Bayesian Quadrature,
since the benefit of second-order information for integration is established:
the Laplace approximation is used to estimate integrals in Bayesian statistics
and relies on a single Hessian observation at the mode of the distribution.

\paragraph{Summary} Bayesian Optimization has proven promising in numerous applications and is an active area of research.
Herein, we provided exact methods 
with an $\mathcal{O}(n^2d)$  MVM complexity 
for kernel matrices arising from $n$ gradient observations in $d$ dimensions
and a large class of kernels, enabling first-order BO to scale to high dimensions. 
In addition,
we derived structures that allow for an $\mathcal{O}(n^2d^2)$ MVM with Hessian kernel matrices,
making future investigations into second-order BO and Bayesian Quadrature possible.

\newpage

\small
\bibliography{arxiv_bo}


\newpage
 
\appendix
\onecolumn

\section{Differentiability of Gaussian Processes}
\label{sec:gp_differentiability}
Summarized in \cite{paci2003thesis}, originally due to Adler.
\begin{theorem}[Mean-Square Differentiability, \citet{adler1981geometry}]
A Gaussian process 
with covariance function $k$
has a mean-square partial derivative at $\bold z$
if and only if 
$[\partial_{x_i} \partial_{y_i} k](\bold z, \bold z)$
 exists and is finite.
\end{theorem}
\begin{proof}
See proof of Theorem 2.2.2 in \cite{adler1981geometry}.
\end{proof}

\begin{theorem}[Sample-Path Continuity, \citet{adler1981geometry}]
A stochastic process $z:\R^d \to \R$
is sample-path continuous if for $\eta > \alpha > 0$, 
\[
E|z(\bold x + \bold r) - z(\bold x)|^\alpha \leq \frac{c\|\bold r\|^{2d}}{|\log \|\bold r\||^{1+\eta}}.
\]
\end{theorem}
\begin{proof}
See Corrollary of Theorem 3.2.5 in \cite{adler1981geometry}.
\end{proof}

\begin{theorem}[Sample-Path Continuity, \citet{adler1981geometry}]
A Gaussian process $z:\R^d \to \R$
is sample-path continuous if for $\eta > \alpha > 0$, 
\[
E|z(\bold x) - z(\bold y)|^2 \leq \frac{c}{|\log \|\bold x - \bold y\||^{1+\eta}}.
\]
\end{theorem}
\begin{proof}
See Corrollary of Theorem 3.2.5 in \cite{adler1981geometry}.
\end{proof}

The following result states that every positive-definite isotropic kernel can be expressed as a scale-mixture of Gaussian kernels,
which aids the derivation of their differentiability properties.
In particular,
\begin{theorem}[\citet{schoenberg1938metric}]
\label{theorem:schoen}
Suppose an isotropic kernel function $k$ is positive-definite on a Hilbert space.
Then there is a non-decreasing and bounded $H$ such that
\begin{equation}
k(\tau) = \int_0^\infty \exp(-\tau^2 s) dH(s).
\end{equation}
We refer to $s$ as the scale parameter.
\end{theorem}
\begin{proof}
This is due to Theorem 2 of \citet{schoenberg1938metric}.
\end{proof}

Sample function (or almost sure) differentiability is a stronger property
that requires a more subtle analysis. 
\citet{paci2003thesis} provided the following result guaranteeing path differentiability 
for isotropic kernels.

\begin{theorem}[Sample-Path Differentiability, \citet{paci2003thesis}]
Consider a Gaussian process with an isotropic covariance function $k$,
and suppose $H(s)$ is related to $k$ as in Theorem~\ref{theorem:schoen}.
If the $2m$ moments $(E_H[s], ..., E_H[s^{2m}])$ of the scale parameter $s$ are finite,
the process is $m$ times sample-path differentiable.
\end{theorem}
\begin{proof}
This is essentially Theorem 10 in \cite{paci2003thesis}.
\end{proof}

\citet{paci2003thesis} further uses this result to prove that the exponentiated quadratic and rational quadratic 
kernels are infinitely and Mat\'ern kernels are finitely sample-path differentiable. 
A number of kernels like the exponential and the Mat\'ern-1/2 kernels
do not give rise to differentiable paths and can thus not be used in conjunction with gradient information.
Notably, even the Mat\'ern-3/2 kernel, which has a differentiable mean function, 
does not give rise to differentiable paths.

\section{Explicit Derivation of Gradient Structure of Neural Network Kernel}
\label{sec:nn}
The point of this section is to show an explicit derivation of the gradient structure 
of the neural network kernel,
{\it which is obviated by the structure-deriving AD engine proposed in this work.} 

Another interesting class of kernels are those that arise from analyzing the limit of certain neural network architectures as the number of hidden units tends to infinity.
It is known that a number of neural networks converge to a Gaussian process under such a limit.
For example, if the error function is chosen as the non-linearity for a neural network with one hidden layer,
the kernel of the limiting process has the following form:
\[
k_{NN}(\bold x, \bold y) \eqdef \sin^{-1} \left(
	\frac{\bold x^\top \bold y}{
	\sqrt{n(\bold x) n(\bold y)}
	} \right),
\]
where $n(\bold x) \eqdef 1+\bold x^\top \bold x$.
Formally, this is similar but not equivalent to the inner product kernels discussed above.
The kernel gives rise to the following more complex gradient kernel structure
\[
[\nabla_{\bold x} k](\bold x, \bold y) = \tilde f'(r) \left(\bold y - \frac{r}{n(\bold x)} \bold x \right),
\]
where $\tilde f'(r) \eqdef f'(r) / \sqrt{n(\bold x) n(\bold y)}$, and
\[
[\bold G k](\bold x, \bold y) = 
	\tilde f'(r) \bold I_d +
	\begin{bmatrix} \bold x & \bold y \end{bmatrix}
	\bold A
	 \begin{bmatrix} \bold x & \bold y \end{bmatrix}^\top,
\]
where
\[
\begin{aligned}
\bold A \eqdef 
	\begin{bmatrix} -\frac{g(r)}{n(\bold x)} & \frac{g(r) r}{n(\bold x) n(\bold y)} \\
	\tilde f''(r) & -\frac{g(r)}{n(\bold y)}
	\end{bmatrix}, \qquad
\tilde f''(r) \eqdef \frac{f''(r)}{n(\bold x) n(\bold y)}, \qquad \text{and} \qquad
g(r) \eqdef \left(\tilde f'(r) + \tilde f''(r) r \right).
\end{aligned}
\]
Notably, this is a rank-two correction to the identity, compared to the rank-one corrections for isotropic and dot-product kernels above.

\section{Hessian Structure}
\label{sec:hessian_appendix}

Note that for arbitrary vectors $\bold a, \bold b$, not necessarily of the same length,
$\bold a \otimes \bold b = \text{vec}(\bold b \bold a^\top)$. This will come in handy to simplify certain expressions in the following.

\paragraph{Dot-Product Kernels}
First, note that
\[
\nabla_{\bold y}^\top \text{vec}(\bold y \bold y^\top) = 
\bold I_d \otimes \bold y + \bold y \otimes \bold I_d
\qquad 
\nabla_{\bold y} \nabla_{\bold y}^\top \text{vec}(\bold y \bold y^\top) = 
\bold S_{dd} + \bold I_{d^2}.
\]
Where $\bold S_{dd}$ is a "shuffle" matrix such that 
$\bold S_{dd} \text{vec}(\bold A) = \text{vec}(\bold A^\top)$,
and 
for square matrices $\bold A \in \R^{n\times n}$ and $\bold B \in \R^{m \times m}$, the Kronecker sum is defined as $\bold A \oplus \bold B
\eqdef
\bold A \otimes \bold I_m + \bold I_n \otimes \bold B$.
Then for dot-product kernels, we have
\[
[\bold h_{\bold x}k](\bold x, \bold y) = f''(r) \text{vec}(\bold y \bold y^\top).
\]
\[
[\bold h_{\bold x} \nabla_{\bold y}^\top k](\bold x, \bold y) = f''(r) (\bold I_d \otimes \bold y + \bold y \otimes \bold I_d) 
+ f'''(r) \text{vec}(\bold y \bold y^\top) \bold x^\top.
\]
\[
[\bold h_{\bold y}^\top \bold h_{\bold x}k](\bold x, \bold y) = 
(\bold I_{d^2} + \bold S_{dd}) 
[f''(r) \bold I_{d^2} + f'''(r) (\bold y \bold x^\top \oplus \bold y \bold x^\top)] 
+ f''''(r) \text{vec}(\bold y \bold y^\top) \text{vec}(\bold x \bold x^\top)^\top.
\]

\paragraph{Isotropic Kernels}

Then for isotropic product kernels with $r = \|\bold r\|^2_2$, we have
\[
\bold J_{\bold x} \text{vec}(\bold r \bold r^\top) = 
\bold I_d \otimes \bold r + \bold r \otimes \bold I_d
\qquad
\bold H_\bold y \text{vec}(\bold r \bold r^\top) = 
\bold S_{dd} + \bold I_{d^2}.
\]
Which implies
\[
[\bold h_{\bold x}k](\bold x, \bold y) = f'(r) \text{vec}(\bold I_{d}) + f''(r) \text{vec}(\bold r \bold r^\top).
\]
\[
[\bold h_{\bold x} \nabla_{\bold y}^\top k](\bold x, \bold y) = 
-f''(r) 
(\bold I_d \otimes \bold r + \bold r \otimes \bold I_d) 
- [f''(r) \text{vec}(\bold I_{d}) + f'''(r) \text{vec}(\bold r \bold r^\top)]
\bold r^\top.
\]

\[
\begin{aligned}
\bold h_{\bold y}^\top \bold h_{\bold x}k (\bold x, \bold y) &= 
(\bold I_{d^2} + \bold S_{dd}) [f''(r) \bold I_{d^2} + 
f'''(r) (\bold r \bold r^\top \oplus \bold r \bold r^\top) ] \\
&+ \begin{bmatrix}
\text{vec}(\bold I_{d}) & \text{vec}(\bold r \bold r^\top)
\end{bmatrix}
\begin{bmatrix}
f''(r) & f'''(r)\\
f'''(r) & f''''(r)
\end{bmatrix}
\begin{bmatrix}
\text{vec}(\bold I_{d}) & \text{vec}(\bold r \bold r^\top)
\end{bmatrix}^\top.
\end{aligned}
\]

\paragraph{A Chain Rule}

$k(\bold x, \bold y) = (f \circ g)(\bold x, \bold y)$.
\[
[\bold h_{\bold x}k](\bold x, \bold y) = f'(r) \bold h_{\bold x} [g] + f''(r) \text{vec}(\nabla_{\bold x} g \nabla_{\bold x} g^\top).
\]
\[
[\bold h_{\bold x} \nabla_{\bold y}^\top k](\bold x, \bold y) = 
f''(r) 
( \bold H_{\bold x} g \otimes \nabla_\bold y g + \nabla_\bold y g \otimes \bold H_\bold x g) 
+ [f''(r) \bold h_{\bold x} [g]) + f'''(r) \text{vec}(\nabla_\bold x g \nabla_\bold x g^\top)]
\nabla_\bold y g^\top.
\]

\[
\begin{aligned}
\bold h_{\bold x} \bold h_{\bold y}^\top k (\bold x, \bold y) &= 
(\bold I_{d^2} + \bold S_{dd}) [f''(r) \bold I_{d^2} + 
f'''(r) (\nabla_\bold x g \nabla_\bold x g^\top \oplus \nabla_\bold y g \nabla_\bold y g^\top) ] \\
&+ \begin{bmatrix}
\bold h_{\bold x} g & \text{vec}(\nabla_\bold x g \nabla_\bold x g^\top)
\end{bmatrix}
\begin{bmatrix}
f''(r) & f'''(r)\\
f'''(r) & f''''(r)
\end{bmatrix}
\begin{bmatrix}
\bold h_{\bold y} g & \text{vec}(\nabla_\bold y g \nabla_\bold y g^\top)
\end{bmatrix}^\top.
\end{aligned}
\]

\paragraph{Vertical Scaling}
$k(\bold x, \bold y) = f(\bold x) h(\bold x, \bold y) f(\bold y)$ for a scalar-valued $f$, then

\[
\begin{aligned}
\bold h_{\bold x} k(\bold x, \bold y) =& 
	\bold h_{\bold x} [f(\bold x) h(\bold x, \bold y)] f(\bold y) \\
	=& \ [f(\bold x) \bold h_{\bold x} [h](\bold x, \bold y) \\
	& + \bold h[f](\bold x) h(\bold x, \bold y) \\
	& + \nabla_{\bold x} [h](\bold x, \bold y) \otimes \nabla[f](\bold x) \\
	& + \nabla[f](\bold x) \otimes \nabla_{\bold x}[h] (\bold x, \bold y)] \ f(\bold y) \\
[\bold h_{\bold x} \nabla_{\bold y}^\top k](\bold x, \bold y)
	=& \ [f(\bold x) [\bold h_{\bold x} \nabla_{\bold y}^\top h](\bold x, \bold y) \\
	& +\bold h[f](\bold x) [\nabla_{\bold y}^\top h](\bold x, \bold y)\\
	& + \bold G[h] (\bold x, \bold y) \otimes \nabla[f](\bold x) \\
	& + \nabla[f](\bold x) \otimes \bold G[h](\bold x, \bold y)] \ f(\bold y)\\
	& + \bold h_{\bold x} [f(\bold x) h(\bold x, \bold y)] \nabla_{\bold y}^\top \ f(\bold y)\\
	[\bold h_{\bold x} \bold h_{\bold y}^\top k](\bold x, \bold y) 
	=& \ [f(\bold x) [\bold h_{\bold x} \bold h_{\bold y}^\top h](\bold x, \bold y) \\
	& +\bold h[f](\bold x) [\bold h_{\bold y}^\top h](\bold x, \bold y)\\
	& + \bold G[h] (\bold x, \bold y) \otimes \nabla[f](\bold x) \nabla^\top[f](\bold y) \\
	& + \nabla[f](\bold x) \nabla^\top[f](\bold y) \otimes \bold G[h](\bold x, \bold y)] \ f(\bold y)\\
	& + \bold h_{\bold x} [f(\bold x) h(\bold x, \bold y)] \bold h_{\bold y}^\top \ f(\bold y)\\
\end{aligned}
\]
Again, we observe a structured representation of the Hessian-kernel elements which permit a multiply in $\mathcal{O}(d^2)$ operations.

\paragraph{Warping}
$k(\bold x, \bold y) = h(\bold u(\bold x), \bold u(\bold y))$,

\[
\begin{aligned}
\bold h_{\bold x} k(\bold x, \bold y) 
	&= (\bold J \otimes \bold J)^\top[\bold u](\bold x) \ 
	[\bold h_{\bold x} h](\bold u(\bold x), \bold u(\bold y)) \\
[\bold h_{\bold x} \nabla_{\bold y}^\top k](\bold x, \bold y) 
	&= (\bold J \otimes \bold J)^\top[\bold u](\bold x) \ 
		[\bold h_{\bold x} \nabla_{\bold y}^\top h](\bold u(\bold x), \bold u(\bold y)) \
		\bold J[\bold u](\bold y) \\
[\bold h_{\bold x} \bold h_{\bold y}^\top k](\bold x, \bold y) 
	&= (\bold J \otimes \bold J)^\top[\bold u](\bold x) \
	[\bold h_{\bold x} \bold h_{\bold y}^\top h](\bold u(\bold x), \bold u(\bold y)) \
	(\bold J \otimes \bold J)[\bold u](\bold y).
\end{aligned}
\]

We therefore see that $\bold K^{\bold H} = \bold h_{\bold x} \bold h_{\bold y}^\top k(\bold X) = 
\bold D_{\bold J} [\bold h_{\bold x} \bold h_{\bold y}^\top h](\bold X) \bold D_{\bold J}$,
where $\bold D_{\bold J}$ is the block-diagonal matrix whose $i^\text{th}$ block is equal to $(\bold J \otimes \bold J)[\bold u](\bold x_i) = \bold J[\bold u](\bold x_i) \otimes \bold J[\bold u](\bold x_i)$.
Note that for linearly warped kernels for which $\bold u(\bold x) = \bold U \bold x$,
where $\bold U \in \R^{r \times d}$, 
we have $(\bold J \otimes \bold J)[\bold u](\bold x_i) = \bold U \otimes \bold U$
so that we can multiply with the kernel matrix $\bold K^{\bold H}$ in $\mathcal{O}(n^2r^2 + n(d^2r + r^2d))$. 
The complexity is due to the following property of Kronecker product: 
\[
(\bold U \otimes \bold U) \text{vec}(\bold H) = \text{vec}(\bold U \bold H \bold U^\top),
\]
which can be computed in $\mathcal{O}(d^2r + r^2d)$ 
for every of the $n$ Hessian observations.

\section{Combining Derivative Orders}
\label{sec:combiningorders}
Combining observations of the function values and its first and second derivatives is straightforward via the following block-structured kernel:
\[
\begin{bmatrix}
k& \nabla_{\bold y}[k]^\top & \bold h_{\bold y}[k]^\top \\
 \nabla_{\bold x} [k] & \bold G [k] & \bold J_\bold x [\bold h_{\bold y}[k]] \\
 \bold h_{\bold x}[k] & \bold J_\bold y [\bold h_{\bold x}[k]] & \bold H[k] 
\end{bmatrix}.
\]
If all constituent blocks permit a fast multiply - $\mathcal{O}(d)$ for gradient and $\mathcal{O}(d^2)$ for Hessian-related blocks - the entire structure permits a $\mathcal{O}(d^2)$ multiply,
even though the na\"ive cost is $\mathcal{O}(d^4)$. 
If only value and gradient observations are required, only the top-left two-by-two block is necessary,
which can be carried out in $\mathcal{O}(d)$ in the structured case
and which we implemented as the \href{https://github.com/SebastianAment/CovarianceFunctions.jl/blob/c551902160963980d35e0d17f693f481d0e33d70/src/gradient.jl#L400}{\texttt{ValueGradientKernel}}.

\paragraph{Discussion}
Recall that the computational complexity of multiplying with the gradient and Hessian kernel matrices is $\mathcal{O}(n^2d)$ and $\mathcal{O}(n^2d^2)$, respectively.
Thus, the gradient-based method can only make a factor of $\mathcal{O}(\sqrt{d})$ more observations than the Hessian-based method for the same computational cost.
Therefore, it is computationally easier to incorporate additional information at a single point 
than it is to combine first-order information at several points.
Since the Hessian contains $d$ times more pieces of information than the gradient,
the former could be more efficient in certain scenarios.
We compare the scaling of the multiplications arising from first- and second-order data experimentally in Section \ref{sec:scaling}
but leave a comprehensive comparison of first-order and second-order BO to future work.
The main goal of the current work is to enable such investigations in the first place
by providing the required theoretical advances and practical infrastructure, see \href{https://github.com/SebastianAment/CovarianceFunctions.jl}{\texttt{CovarianceFunctions.jl}}.

\section{Accuracy Comparison to D-SKIP}
 \label{sec:accuracy_dskip}
 
 Figure~\ref{fig:mvm_accuracy_comparisons} shows the relative accuracy of MVMs using D-SKIP and our work for RBF gradient kernel matrices with $n = 1024$.
 Note that the indicated asymptotic lines $\mathcal{O}(\sqrt{d})$ and $\mathcal{O}(d^4)$ are
 there for comparison only and do not indicate a theoretically expected error scaling.
 
 \begin{figure}[t!]
  \centering
\includegraphics[width=.45\textwidth]{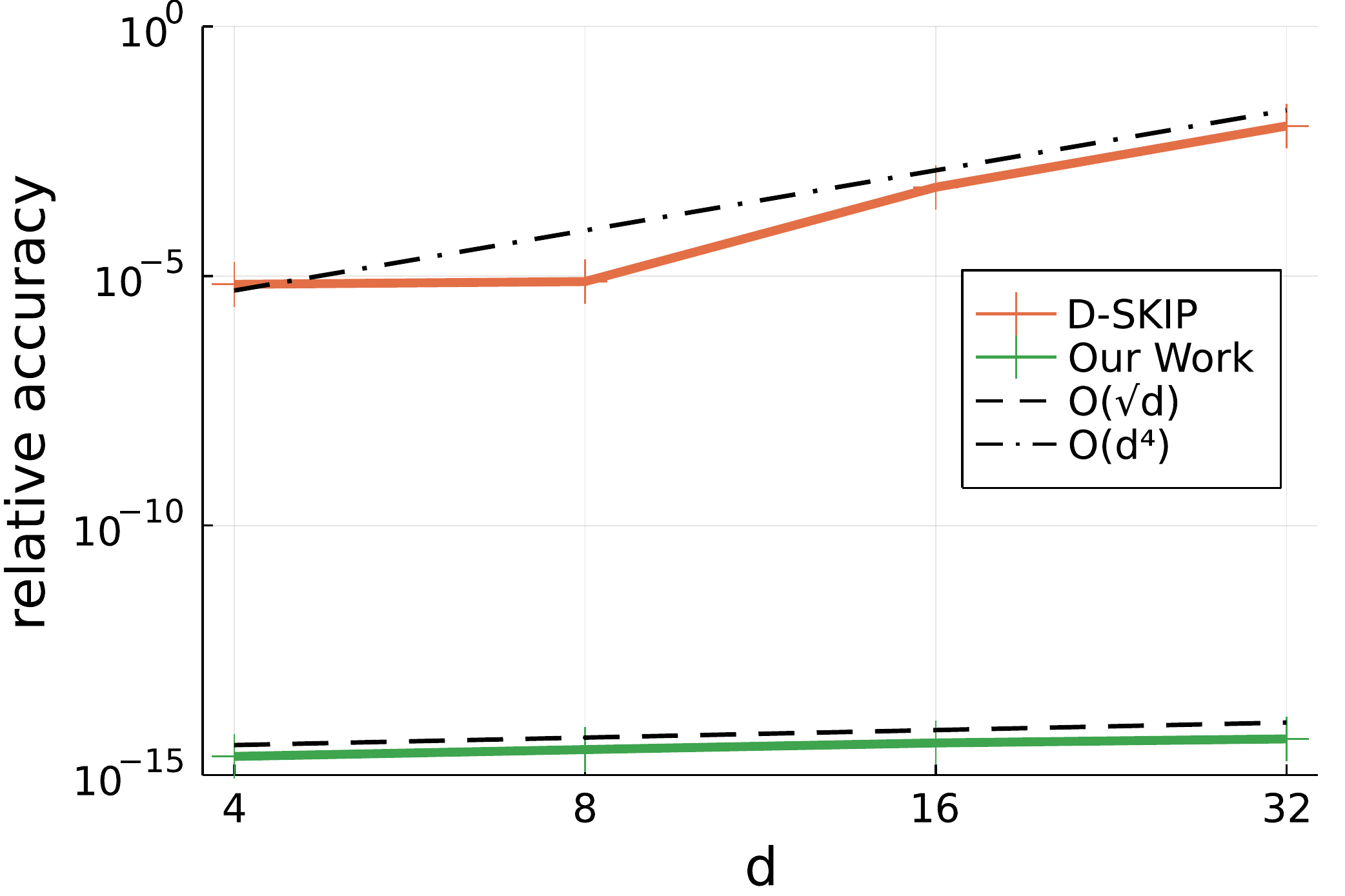}
  \caption{Relative accuracy of MVMs using D-SKIP and our work for RBF gradient kernel matrices with $n = 1024$.}
  \label{fig:mvm_accuracy_comparisons}
\end{figure}

\section{Test Functions for Bayesian Optimization}
\label{sec:test_appendix}

\paragraph{Rosenbrock}
\[
 \text{Rosenbrock}_d(\bold x) = \sum_{i}^{d-1} (x_i - a)^2 + b (x_{i+1} - x_{i}^2)^2 
\]
For our experiments, we let $a = 0$, $b = 10$ and evaluate it on $x_i \in [-3, 3]$.

\paragraph{Ackley} 
 The Ackley function in $d$ dimensions is given by
 \[
 \text{Ackley}_d(\bold x) = 
 \exp(1) + a - a \exp \left( -b \|\bold x\|_2 / \sqrt{d} \right) - \exp \left( \sum_{i=1}^d \cos(cx_i) / d \right),
 \]
where $a = 20, b=0.2$, and $c=2\pi$.
The function is usually evaluated on the hypercube $x_i \in [-32.768, 32.768]$ and has a single global minimum at the origin.
We evaluated it on $[-10, 10]$ for our experiments.

\paragraph{Rastrigin} 
The Rastrigin function is defined by 
\[
\text{Rastrigin}_d(\bold x) = 
10d + \sum_{i=1}^d [x_i^2 - 10\cos(2\pi x_i)],
\]
is usually evaluated on the hypercube $x_i \in [-5.12, 5.12]$ 
and has a single global minimum of $0$ at the origin.

\paragraph{Drop-Wave}
The Drop-Wave function is defined by
\[
\text{DropWave}_d(\bold x) = 
-\frac{1+\cos(12 \|\bold x\|_2)}{\|\bold x\|_2^2 + 2},
\]
is usually evaluated on the two-dimensional square  $x_i \in [-5.12, 5.12]$ and has a single global minimum of $-1$ at the origin. 
We note however, that it is straightforwardly generalized to arbitrary input dimensions, since it only depends on the Euclidean norm of the input.

\paragraph{Griewank}
The Griewank function is defined by
\[
\text{Griewank}_d(\bold x) = \|\bold x\|^2_2 / 4000 - \prod_i \cos \left(\frac{x_i} {\sqrt{i}} \right) + 1
\]
is usually evaluated on the two-dimensional square  $x_i \in [-600, 600]$ and has a single global minimum of $0$ at the origin. 
We note however, that it is straightforwardly generalized to arbitrary input dimensions, since it only dependents on the Euclidean norm of the input.
We evaluated it on $x_i \in [-200, 200]$ in our experiments.

\section{Theoretical Background for Global and Bayesian Optimization}
\label{sec:optimization_theory}
\paragraph{Global Optimization} 
\citet{torn1989global} noted that the global approximation and optimization of continuous functions 
is a hard problem in general.
They proved that, for any algorithm, 
there are multimodal functions that give rise to an arbitrarily large error after a finite number of samples,
and that a global optimization algorithm converges to 
the optimum of any continuous function on a compact set, 
{\it if and only if it eventually samples the set densely}.

If one imposes certain structure on the set of functions,
finite-sample guarantees become feasible.
For example, function whose modulus of continuity $\omega(\delta) = \max_{d(x, y) < \delta} |f(x) - f(y)|$ is bounded, like Lipschitz-continuous functions, admit the following finite-sample bound.
Let $A$ be the compact set on which the function $f$ is to be minimized
and $\{x_i\} \subset A$ be the observed samples, then
\[
\min_{1\leq i \leq n} f(x_i) - \min_{x \in A} f(x) < \omega(d_n),
\]
where $d$ is the metric on $A$ and
$d_n = \max_{x \in A} \min_{1\leq i \leq n} d(x, x_i)$ is the dispersion of the samples.
By extension, this implies that continuously differentiable functions on a compact set admit a similar bound, since the suprema of their derivatives are attained and finite.

\paragraph{Bayesian Optimization} 
Regarding Bayesian optimization algorithms, \citet{torn1989global} noted
that 
``even if [BO] is very attractive theoretically it is too complicated for algorithmic realization. 
Because of the fairly cumbersome computations involving operations with the inverse of the {\it covariance matrix} and complicated auxiliary optimization problems.
The resort has been to use simplified models.''
The techniques put forth in Section~\ref{sec:methods} of this article make significant strides in reducing this complexity.
\citet{bull2011convergence} provides theoretical results for the convergence of 
Bayesian optimization algorithms based on 
Gaussian processes and the Expected Improvement (EI) acquisition function.
The results hold for noiseless observations of the function to be optimized.
\citet{srinivas2012information}
proved regret bounds for the multi-armed bandit problem for which the payoff function is
drawn from a GP with particular kernel functions and the upper-confidence bound acquisition function.
\citet{shekhar2021significance} recently proved that zeroth-order BO has a regret lower bound that increases exponentially with the dimension,
while first-order BO can achieve a much better regret bound of $\mathcal{O}(d \log^2 n)$,
where $d$ is the dimensionality of the input and $n$ is the number of observations,
using a two-stage procedure: the first stage identifies a locally quadratic neighborhood around a presumed optimum, while the second stage takes local gradient steps based on stochastic estimates of the gradient.

\end{document}